\documentclass[10pt,journal,compsoc]{IEEEtran}
\ifCLASSOPTIONcompsoc
  \usepackage[nocompress]{cite}
\else
  \usepackage{cite}
\fi
\ifCLASSINFOpdf
\else
\fi
\usepackage{url}
\usepackage{graphicx}
\graphicspath{{pic/}}
\label{package}
\usepackage{subfigure}
\usepackage{epstopdf}
\usepackage{multirow}
\usepackage{multicol}
\usepackage{booktabs}
\usepackage{algorithm}
\usepackage{algorithmic}

\newtheorem{myDef}{Definition} 
\newtheorem{myTheo}{Theorem}
\newenvironment{proof}{{\noindent\it Proof.}\quad}{\hfill $\square$\par}
\usepackage{flushend}
\usepackage{amsmath,amssymb,amsfonts}
\usepackage{adjustbox}
\usepackage{pdflscape}

\begin{document}
%
% paper title
% Titles are generally capitalized except for words such as a, an, and, as,
% at, but, by, for, in, nor, of, on, or, the, to and up, which are usually
% not capitalized unless they are the first or last word of the title.
% Linebreaks \\ can be used within to get better formatting as desired.
% Do not put math or special symbols in the title.
\title{Ball $k$-means}
%
%
% author names and IEEE memberships
% note positions of commas and nonbreaking spaces ( ~ ) LaTeX will not break
% a structure at a ~ so this keeps an author's name from being broken across
% two lines.
% use \thanks{} to gain access to the first footnote area
% a separate \thanks must be used for each paragraph as LaTeX2e's \thanks
% was not built to handle multiple paragraphs
%
%
%\IEEEcompsocitemizethanks is a special \thanks that produces the bulleted
% lists the Computer Society journals use for "first footnote" author
% affiliations. Use \IEEEcompsocthanksitem which works much like \item
% for each affiliation group. When not in compsoc mode,
% \IEEEcompsocitemizethanks becomes like \thanks and
% \IEEEcompsocthanksitem becomes a line break with idention. This
% facilitates dual compilation, although admittedly the differences in the
% desired content of \author between the different types of papers makes a
% one-size-fits-all approach a daunting prospect. For instance, compsoc 
% journal papers have the author affiliations above the "Manuscript
% received ..."  text while in non-compsoc journals this is reversed. Sigh.
\author{Shuyin Xia, Daowan Peng, Deyu Meng, Changqing Zhang, Guoyin Wang, Zizhong Chen, Wei Wei% <-this % stops a space
	\IEEEcompsocitemizethanks{\IEEEcompsocthanksitem Shuyin Xia, Daowan Peng, and Guoyin Wang are with the Department
		of Chongqing Key Laboratory of Computational Intelligence, Chongqing University of Posts and Telecommunications, Chongqing 400065, China.\protect\\
		% note need leading \protect in front of \\ to get a newline within \thanks as
		% \\ is fragile and will error, could use \hfil\break instead.
		E-mail: xiasy@cqupt.edu.cn, daowan\_peng@qq.com, wanggy@cqupt.edu.cn
		\IEEEcompsocthanksitem Deyu Meng is with National Engineering Laboratory for Algorithm and Analysis Technologiy on Big Data, Xi’an Jiaotong University,Xi'an 710049, China.\protect\\
		E-mail: dymeng@xjtu.edu.cn
		\IEEEcompsocthanksitem Changqing Zhang is with College of Intelligence and Computing, Tianjin University, 300072, China.\protect\\
		E-mail:zhangchangqing@tju.edu.cn
		\IEEEcompsocthanksitem Zizhong Chen is with Department of Computer Science and Engineering, University of California, Riverside,900 University Avenue, Riverside, CA 92521, USA.\protect\\
		E-mail: chen@cs.ucr.edu
		\IEEEcompsocthanksitem Wei Wei is with School of Computer Science and Engineering, Xi'an University of Technology. Xi'an 710048, China. \protect\\E-mail: weiwei@xaut.edu.cn.
	}% <-this % stops a space
	%\thanks{Manuscript received April 19, 2005; revised August 26, 2015.}
}
\IEEEtitleabstractindextext{%
\begin{abstract}
This paper presents a novel accelerated exact $k$-means algorithm called the Ball $k$-means algorithm, which uses a ball to describe a cluster, focusing on reducing the point-centroid distance computation. The Ball $k$-means can accurately find the neighbor clusters for each cluster resulting distance computations only between a point and its neighbor clusters’ centroids instead of all centroids. Moreover, each cluster can be divided into a stable area and an active area, and the later one can be further divided into annulus areas. The assigned cluster of the points in the stable area is not changed in the current iteration while the points in the annulus area will be adjusted within a few neighbor clusters in the current iteration. Also, there are no upper or lower bounds in the proposed Ball $k$-means. Furthermore, reducing centroid-centroid distance computation between iterations makes it efficient for large k clustering. The fast speed, no extra parameters and simple design of the Ball $k$-means make it an all-around replacement of the naive $k$-means algorithm.
\end{abstract}

% Note that keywords are not normally used for peerreview papers.
}

% make the title area
\maketitle

% To allow for easy dual compilation without having to reenter the
% abstract/keywords data, the \IEEEtitleabstractindextext text will
% not be used in maketitle, but will appear (i.e., to be transported)
% here as \IEEEdisplaynontitleabstractindextext when compsoc mode
% is not selected <OR> if conference mode is selected - because compsoc
% conference papers position the abstract like regular (non-compsoc)
% papers do!
\IEEEdisplaynontitleabstractindextext
% \IEEEdisplaynontitleabstractindextext has no effect when using
% compsoc under a non-conference mode.

% For peer review papers, you can put extra information on the cover
% page as needed:
% \ifCLASSOPTIONpeerreview
% \begin{center} \bfseries EDICS Category: 3-BBND \end{center}
% \fi
%
% For peerreview papers, this IEEEtran command inserts a page break and
% creates the second title. It will be ignored for other modes.
\IEEEpeerreviewmaketitle

\label{the description of dimesion}

\section{Ball $k$-means Clustering}

In this section, the main idea of the Ball k-mean algorithm is presented by introducing the ball cluster concept, neighbor clusters searching, ball cluster division, and the mechanism on how to reduce centroid-centroid distance computation between iterations.

\subsection{Ball Cluster Concept }

A ball structure is characterized by a radius and centroid. Therefore, to describe a cluster and analyze the proposed method, we propose the ball cluster concept, where a ball is used to describe a cluster.

​
\begin{myDef}
	
	Given a cluster $C$, $C$ is called as a ball cluster that is defined by its centroid $c$ and radius $r$ as follows:

	\begin{equation}
	c = \frac{1}{|N|} \sum_{i=1}^{N} x_i, r = \max (\|x_i - c\|),
	\end{equation}
	
	where $x_i$ denotes a point assigned to $C$, and $|N|$ denotes the number of samples in $C$.
\end{myDef}

\subsection{Neighbor Cluster Searching}
To skip the distance calculation between a point and a centroid that is very far away from that point, we introduce a method that can find the neighbor clusters of each cluster. Thus, the distance computation is limited to the points and their neighbor clusters. The neighbor cluster is defined by Definition 2.

\begin{myDef}
	
	Given two ball clusters $C_i$ and $C_j$, whose centroids are denoted as $c_i$ and $c_j$; $r_i$ represents the radius of $C_i$, if it satisfies the following inequality: 
	
	\begin{equation}
	r_i > \frac{1}{2}\left\| c_i - c_j \right\|,
	\end{equation}
	then, $C_j$ is a neighbor cluster of $C_i$. 
\end{myDef}

Equation (2) indicates the neighbor relationship is not symmetric.
Specifically, for two ball clusters $C_i$ and $C_j$, referring to Fig. 1, their neighbor relationship can be one of the three following types:

(1) $C_i$ and $C_j$ are neighbor clusters, and vice versa, such as clusters $C_2$ and $C_3$ presented in Fig. 1. As $C_2$ and $C_3$ are neighbor clusters, some points in $C_3$ ($C_2$) may be adjusted into $C_2$ ($C_3$) in the current iteration.

(2) $C_i$ is a neighbor cluster of $C_j$, but $C_j$ is not a neighbor cluster of $C_i$; such as $C_1$ and $C_3$ presented in Fig. 1; namely, $C_1$ is a neighbor cluster of $C_3$, but $C_3$ is not a neighbor cluster of $C_1$. Therefore, some points in $C_3$ may be adjusted into $C_1$ but no points in $C_1$ can be adjusted into $C_3$.

(3) $C_i$ and $C_j$ are not neighbor clusters, such as $C_3$ and $C_4$ presented in Fig. 1. Therefore, points in $C_3$ ($C_4$) cannot be adjusted into $C_4$ ($C_3$) in the current iteration. 

\begin{figure}
	\centering
	\includegraphics[width = .47\textwidth]{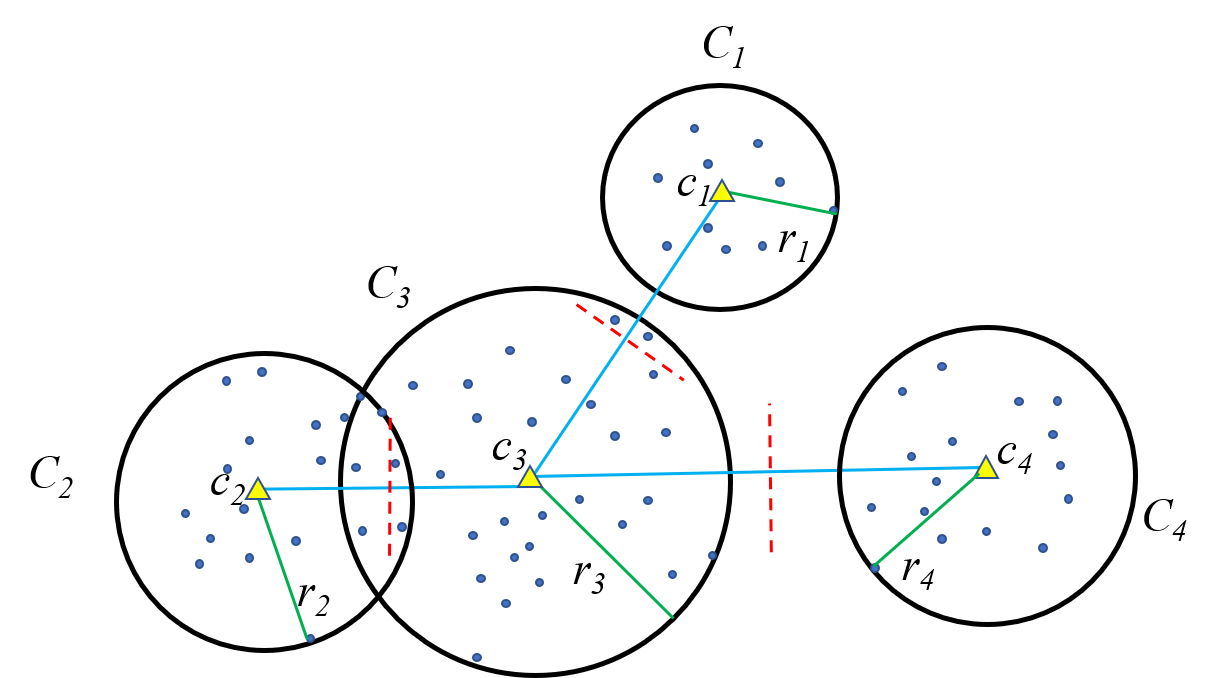}
	\caption{The schematic diagram of neighbor relationship of the queried ball cluster $C_3$. The red dash line represents the vertical bisector of the centroids of two ball clusters. The yellow triangle and green line represent the centroid and radius of a cluster respectively.
%	 $C_2$ and $C_3$ are neighbor clusters with each other. $C_1$ is a neighbor cluster of $C_3$, but $C_3$ is not  of $C_1$. $C_3$ and $C_4$ are not neighbor cluster with each other.
	}
	\label{Fig2}
\end{figure}

\begin{myTheo}
	
    Given two clusters $C_i$ and $C_j$ with centroids $c_i$ and $c_j$, respectively. For a queried ball cluster $C$ with a centroid and radius denoted as $c$ and $r$, if $C_i$ is a neighbor cluster of $C$ (i.e.,$r>\frac{1}{2}\|c-c_i\|$) while $C_j$ is not a neighbor cluster of $C$(i.e., $r\leq\frac{1}{2}\|c-c_j\|$), then, it holds that some points in $C$ may be adjusted into $C_i$, while all points in $C$ cannot be adjusted into $C_j$.
	
\end{myTheo}

\label{old proof}	
%	Supposed that $x$ is the point on the segment $cc_i$, then it fulfills (3):
%	
%	\begin{equation}
%	\|c-x\| + \|c_i-x\| = \|c-c_i\|
%	\end{equation}
%	Besides, supposed that $x$ is located in the area constrained by (4): 
%	\begin{equation}	 
%	\|c-x\|>\frac{1}{2} \|c-c_i\|
%	\end{equation}

For a queried cluster $C$, its neighbor ball clusters can be exactly found by Definition 2, so the distance computation of points in $C$ to the centroids of the other clusters is limited to the neighbor clusters of $C$, resulting in a significant decrease in the distance computation amount. In~\cite{ryvsavy2016geometric}, similar method in finding neighbor cluster was proposed. For two clusters $c_i$ and $c_j$, $||c_i-c_j||$ represents the distance between centroids of $c_i$ and $c_j$.  if $m(c_i) + s(c_i) \geq 1/2||c_i-c_j||$, $c_j$ is the neighbor of $c_i$，where $m(c_i)$ represents the radius of $c_i$, and $s(c_i)$ represents half the distance between centroid of $c_i$ and its closest other centroid. On the contrast, in this paper, if $ri >1/2||c_i-c_j||$, $c_j$ is the neighbor of $c_i$, where $r_i$ represents half the distance. Therefore, in comparison with Definition 2, there was one additional element in ~\cite{ryvsavy2016geometric}. Consequently, the condition in ~\cite{ryvsavy2016geometric} was looser than that in Definition 2. In other words, Definition 2 can yield finding less but exacter neighbor clusters than that in~\cite{ryvsavy2016geometric}. In Section 1.3, it is shown that the ball cluster division can further decrease the distance computation amount.

\subsection{Ball Cluster Division}   

A queried ball cluster can be divided into two parts, stable area and active area, which are defined by Definition 3. The points in the stable area stay in the assigned cluster, which is given in Theorem 2. The active area can be further divided into annulus areas as given in Definition 4. Points in each annulus area need to calculate distance only to some of the neighbor clusters, which is given in Theorem 3.

\subsubsection{Stable and Active Areas}
The definitions of the stable area and active area are as follows:

\begin{myDef}	
     Given a queried ball cluster $C_i$, $\left \{ N_{C_i} \right \}$ denotes the centroid set of the neighbor clusters of $C_i$. If $\left \{ N_{C_i}\right \}\neq \varnothing$, for a ball cluster $C_j$ whose centeroid is $c_j$, and $c_j\epsilon \left \{ N_{C_i} \right \}$, then the sphere area whose centeroid and radius $r$ are equal44 to $c_i$ and $\frac{1}{2}min\left ( \|c_i-c_j\| \right )_{c_j\in{ N_{C_i} }} $ respectively is defined as the stable area of $C_i$. And the rest area is defined as the active area of $C_i$.	
\end{myDef}

\begin{myTheo}
	Given a cluster $C_i$, the points in the stable area of $C_i$ cannot be adjusted into any neighbor clusters in the current iteration. 
\end{myTheo}

However, in a special case, when a ball cluster has no neighbor clusters, the stable area is equal to the whole ball cluster. The description similar to the stable area is only provided in ~\cite{elkan2003using}, but it relies on the upper bound which is bigger than the direct distance when checking that filtering condition. On the contrary, Definition 3 provides an exact definition that relies on no bounds. 
%Consequently, the stable area in Definition 3 is exact and larger than that in ~\cite{elkan2003using}, so our method can avoid more distance computation. 

\subsubsection{Active Area Division}

In this section, we show that the active area of a queried cluster can be divided into annulus areas that are generated by the neighbor clusters.
\begin{myDef}Annulus area 
	
	Given a queried ball cluster $C$ with a centroid $c$ and radius $r$, supposing $|\{N_C\}|=k$', $\{N_C\}$ represents the neighbor clusters'centroids set of $C$. $c_{i}$ and $c_{i+1}$ represent the centroids of the $i^{th}$ and $(i+1)^{th}$ closest neighbor clusters of $C$, respectively ($i<k$'). For $x \in C$, the $i^{th}$ annulus area denoted as $\Re_C^{i}$:
	$$\Re_C^{i}=
	\begin{cases}
	\frac{1}{2}\|c-c_{i}\| < \|x-c\|< \frac{1}{2}\|c-c_{i+1}\|,& \text{$i=1...k'$-1}\\
	\frac{1}{2}\|c-c_{i}\| < \|x-c\|<r,& \text{$i=k'$}
	\end{cases}$$
%		
%	The $i$st "annulus area" is a ring generated with its $i$st closest neighbor cluster and ($i+1$)-th closest neighbor cluster as follows: $c$ is the centroid of both the inside circle and the outside circle of the annulus area. The radii of the two circles are equal to $\frac{1}{2}\|c-c_{i}\|$ and $\frac{1}{2}\|c-c_{i+1}\|$ respectively. 
%	If $i=k$', then the i-th "annulus area" is generated by its $i$st closest neighbor cluster and the queried cluster itself.
%	Ai = $\{x \in C: \frac{1}{2}\|c-c_{i}\| < \|x-c\|< \frac{1}{2}\|c-c_{i+1}\|\}$, i=1,2...k'-1
%	Ai = $\{x \in C: \frac{1}{2}\|c-c_{i}\| < \|x-c\|< r\}$, i=k'
\end{myDef}
%\begin{myDef}
%	"annulus area" 
%	
%Given a queried ball cluster $C$ with its centroid denoted as $c$. The number of neighbor clusters of $C$ is denoted as $k$'. $c_{i}$ and $c_{i+1}$ represent the centroids of the i-th closest and (i+1)-th closest neighbor clusters of $C$ respectively ($i < k$'). The area fulfilling the following characters is called as $i$st "annulus area":
%	
%The i-th "annulus area" is a ring generated with its i-th closest neighbor cluster and (i+1)-th closest neighbor cluster as follows: c is the center of both the inside circle and the outside circle of the annulus area. The radii of the two circles are equal to $\frac{1}{2}\|c-c_{i}\|$ and $\frac{1}{2}\|c-c_{i+1}\|$ respectively. 
%	If $i=k$', then the i-th "annulus area" is generated by its $i$st closest neighbor cluster and the queried cluster itself.
%	
%\end{myDef}

\begin{myTheo}
	
Given a queried cluster $C$ with a centroid $c$, supposing $|\{N_C\}|=k'$, the points in its $i^{th}$ annulus area can be adjusted only within its first-$i$ closest neighbor clusters and itself ($i\leq k$').
	
\end{myTheo}

%		\STATE $c_j = mean({x_i | x_i \in C_j})$
%		\ENDFOR
%		\FOR{$j=1,\cdots,k$}
%		\STATE Calculate radius $r_j$ via equation (1).
%		\STATE Set $\{N_{c_j}\} = \emptyset$ be the neighbor clusters' centroids set of $C_j$.
%		\STATE //Calculate the neighbor ball clusters for each cluster.
%		\FOR {$i=1,\cdots,k$}
%		\STATE Compute between-cluster distance for centroids $dist=\|c_i-c_j\|$
%		\IF {$r_j> \frac{1}{2}\|c_i-c_j\|> 0$} 
%		\STATE append $c_i$ to $\{N_{c_j}\}$
%		\ENDIF
%		\ENDFOR
%		\STATE Sort $\{N_{c_j}\}$ with the order from near to far of the distance between $c_j$ and $c_i$.
%		\STATE Filtering the points in "stable area" as defined in Definition 3 ;
%		\FOR {each  point $x_i$ in each "annulus area"}
%		\STATE Calculate the distances from $x_i$ to those neighbor cluster centroids which shown in Theorem 3, and assign $x_i$ to the closest cluster;
%		\ENDFOR
%		\ENDFOR

\label{add iteration}
%\subsection{The Relationship of Neighbor Clusters between Iterations}   
\subsection{Reducing Centroid-centroid Distances Computation between Iterations}
%As we know, two clusters far away to each other cannot be neighbor clusters, so the distance calculation of the centroids between them is redundant which can be avoided.
As presented in Section 1.2, to find the neighbor clusters of each ball cluster, it is needed to calculate all the centroid-centroid distances, which costs $O(k^2)$ per iteration, and for large $k$ clustering, this is a non-negligible cost. In this paper, the purpose of the calculation of centroid-centroid distances is to find the neighbor clusters of the next iteration. If a non-neighbor relationship in the next iteration can be found in advance according to the relationship of the ball clusters in the current iteration, then direct centroid-centroid distance calculation can be avoided. In this work, we develop a method to implement this idea that can find the non-neighbor relationship in advance to avoid unnecessary calculation of centroid-centroid distances. The specific process of this method is formulated as follows.

Let $c_i^{(t)}$ represent the centroid of cluster $C_i$ in the $t^{th}$ iteration, and $\delta(c_i^{(t)})= \| c_i^{(t)}-c_i^{(t-1)} \|$ represent the shift of the cluster centroid of $C_i$ between $(t-1)^{th}$ iteration and $t^{th}$ iteration, and $dist(c_i^{(t)},c_j^{(t)})$ represent the distance of $c_i$ and $c_j$ in the $t^{th}$ iteration.

\begin{myTheo}
	
	Given clusters $C_i$ and $C_j$, supposing that $dist(c_i^{(t-1)},c_j^{(t-1)}) \geq 2r_i^{(t)}+\delta(c_i^{(t)})+\delta(c_j^{(t)})$, then  it holds that $C_j$ cannot be the neighbor cluster of $C_i$ in the current iteration and the centroid-centroid distance of them could be skipped.
\end{myTheo}
\begin{figure}
	\centering
	\includegraphics[width = .47\textwidth]{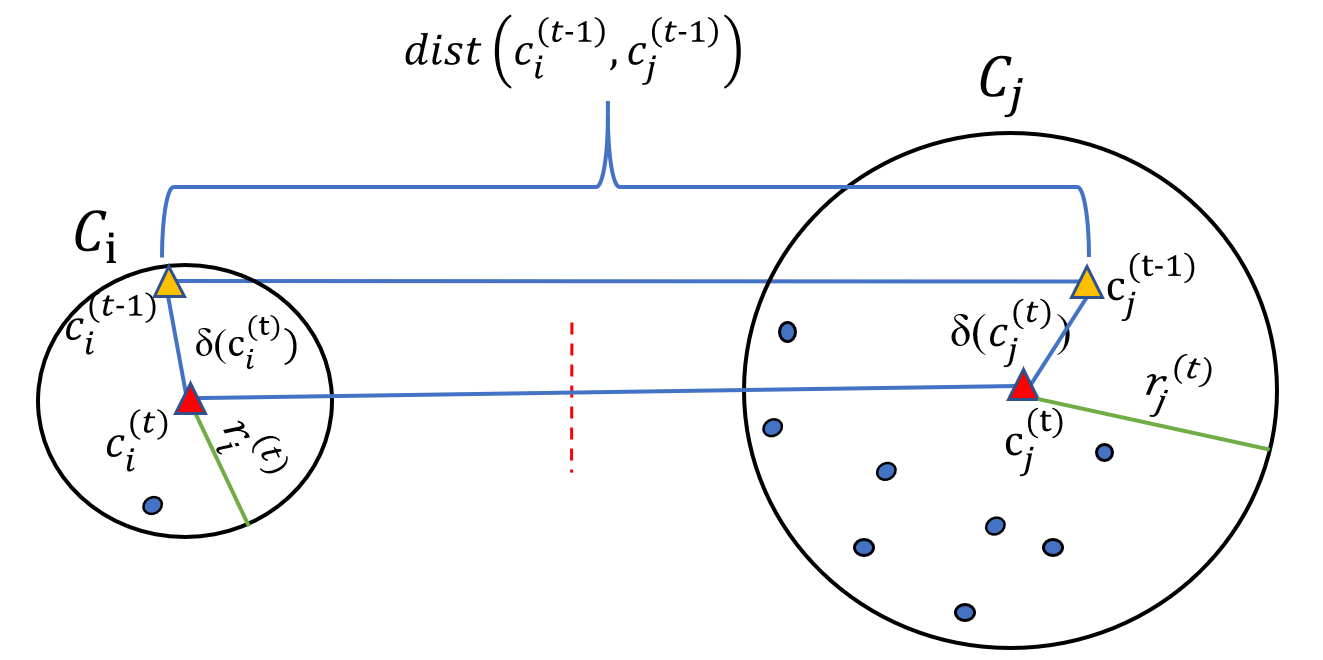}
	\caption{The schematic diagram of  avoiding direct centroid-centroid distance calculation. The red dash line represents the midpoint of $dist(c_i^{(t)},c_j^{(t)})$. $C_j$ is not a neighbor cluster of $C_i$ in the $(t)$-th iteration.
	}
	\label{Fig2}
\end{figure}
\begin{proof}
	With the shift of cluster centroids due to the centroid update, it holds that:\\ $ dist(c_i^{(t)},c_j^{(t)}) \geq dist(c_i^{(t-1)},c_j^{(t-1)}) - \delta(c_i^{(t)}) - \delta(c_j^{(t)})$,
	and the supposing that $dist(c_i^{(t-1)},c_j^{(t-1)}) \geq 2r_i^{(t)}+\delta(c_i^{(t)})+\delta(c_j^{(t)})$,\\
	$\Rightarrow dist(c_i^{(t)},c_j^{(t)})
	\geq 2r_i^{(t)}+\delta(c_i^{(t)})+\delta(c_j^{(t)}) -\delta(c_i^{(t)})-\delta(c_j^{(t)})=2r_i^{(t)}$,\\
	$\Rightarrow dist(c_i^{(t)},c_j^{(t)})
	\geq 2r_i^{(t)}$.
	
	%	 
	%	 \begin{equation}
	%	 
	%	 D(x) =  \begin{cases}
	%	 
	% $(a) dist(c_i^{(t+1)},c_j^{(t+1)}) \geq dist(c_i^{(t)},c_j^{(t)}) - \delta(c_i^{(t)}) - \delta(c_j^{(t)})$, 
	%
	%$(b) dist(c_i^{(t)},c_j^{(t)}) \geq 2r_i^{(t+1)}+\delta(c_i^{(t)})+\delta(c_j^{(t)})$
	%	 
	%	 \end{cases}
	%	 
	%	 \end{equation}

	As given in Definition 2 and Theorem 1, when $2r_i\leq dist(c_i,c_j)$, $C_j$ is not the neighbor of $C_i$. So, as it shows in Fig.2 when $dist(c_i^{(t-1)},c_j^{(t-1)}) \geq 2r_i^{(t)}+\delta(c_i^{(t)})+\delta(c_j^{(t)})$, it holds that $ dist(c_i^{(t)},c_j^{(t)})>2r_i^{(t)}$(i.e., $C_j$ cannot be a neighbor cluster of $C_i$ in the current iteration). Thus, the computation of distance between $c_i$ and $c_j$ can be avoided.

\end{proof}

\subsection{Stable Ball Cluster in Subsequent Iterations}

According to the characteristics of the k-means algorithm, with the number of the iteration increasing, more and more ball clusters tend to be stable, i.e., that the points in it are unchanged. In ball k-means, an stable ball cluster can be simply described as that no points move into this ball cluster and no points in this ball cluster move out in current iteration. Based on this characteristic of the k-means algorithm itself, we propose a method to find those stable ball clusters. In this method, a flag corresponding to a ball cluster is used to judge whether a ball cluster is stable. For a queried ball cluster, if no points in the queried ball cluster move into its nearest cluster, and no points in other ball cluster move into the queried ball cluster in the current iteration, then its flag is marked TRUE. 

%\begin{myTheo}
%	
%	If the flag of the queried ball cluster is marked as TRUE, then in the next iteration, the update step of its centroid of this ball cluster can be avoided. At the same time, the process of establishing the ball cluster model in the current iteration of the ball cluster can also be avoided, i.e., the update step of the radius can be avoided.
%\end{myTheo}
%
%\begin{proof}
%	This proof is straightforward. The flag of the queried ball cluster is TRUE that means no points in this ball cluster were adjusted in last iteration because the points and the centroid of this ball cluster do not change, thus, the radius of this cluster does not change in current iteration. Therefore, the update step of this ball cluster can be avoided.
%\end{proof}

\begin{myTheo}
	Ball k-means is implemented on a given data set $D$. For a queried ball cluster C, if the points in C are not changed, it is called as a stable ball cluster. In one iteration, if all the neighbor ball clusters of C are stable, C will not participate into the distance calculations in next iteration. 
\end{myTheo}

\begin{proof}
	This proof is straightforward. For a queried ball cluster, if all the neighbor ball clusters of C is stable, then the division of the stable area and annulus areas are the same as those in the previous iteration, so the assignment step of the queried ball cluster can be avoided.
\end{proof}

During the iteration of the k-means algorithm, more and more ball clusters will become stable, and the data points in those stable ball cluster will not participate into any distance calculations. Therefore, the time complexity of ball k-means per iteration will become to be sublinear, and the ball k-means will run faster and faster per iteration.

%The Ball $k$-means is presented in Algorithm 1 in detail. The Section 2.4 and 2.5, playing as the process of searching neighbor clusters  and assigning points in clusters of Algorithm 1, is presented in Algorithm 2 and Algorithm3 respectively. 

\end{document}